\documentclass{article}
\usepackage[margin=1in]{geometry} 
\usepackage{amssymb,amsmath,amsthm} 
\usepackage{xcolor} 
\usepackage{verbatim} 
\usepackage[english]{babel} 
\usepackage{blindtext} 
\usepackage[T1]{fontenc} 
\usepackage{underscore} 
\usepackage{enumitem} 
\usepackage{hyperref} 
\usepackage{cleveref} 
\usepackage{booktabs} 
\usepackage{multirow} 
\usepackage{mwe,tikz} 
\usepackage[percent]{overpic} 
\usepackage{kbordermatrix} 
\usepackage{listings} 
\usepackage{xcolor} 
\lstdefinestyle{MATLAB}{ language=Matlab, basicstyle=\ttfamily\scriptsize, stringstyle=\color{magenta}, commentstyle=\color{green!50!black}, numbers=none, frame=single, tabsize=4, breaklines=true, breakatwhitespace=true } 
\makeatletter 
\usetikzlibrary{positioning} 
\usetikzlibrary{automata} 
\usepackage{mathtools} 
\usepackage{graphicx} 
\usepackage{subfigure}
\usepackage{subcaption} 
\usepackage{indentfirst} 
\usepackage{algorithm} 
\usepackage{algorithmic} 
\floatname{algorithm}{Algorithm}

\usepackage[round]{natbib}

\newtheorem{definition}{Definition}[section]

\newtheorem{theorem}{Theorem}[section] 
\newtheorem{lemma}{Lemma}[section]

\newtheorem{example}{Example}[section] 
\numberwithin{equation}{section} 
 
\newcommand{\norm}[1]{\left\lVert #1 \right\rVert}

\newcommand{\RR}{\mathbb{R}} 
\newcommand{\CC}{\mathbb{C}}

\newcommand{\OO}{\mathbb{O}}

\newcommand{\Tr}{{\rm Tr}} 
\newcommand{\loss}{{\rm loss}}

\begin{document} 
\title{Hidden Convexity of Fair PCA and Fast Solver \\ via Eigenvalue Optimization} 
\author{
  Junhui Shen
  \quad 
  Aaron J. Davis
  \quad
  Ding Lu
  \quad
  Zhaojun Bai
}
\date{}
\maketitle 
\noindent

\begin{abstract}
    Principal Component Analysis (PCA) is a foundational technique in machine learning for dimensionality reduction of high-dimensional datasets. However, PCA could lead to biased outcomes that disadvantage certain subgroups of the underlying datasets.
    To address the bias issue, a Fair PCA (FPCA) model was introduced by Samadi et al. (2018) 
    for equalizing the reconstruction loss between subgroups. 
    The semidefinite relaxation (SDR) based approach proposed by 
    Samadi et al. (2018)
    is computationally expensive even for suboptimal solutions. 
    To improve efficiency,
    several alternative variants of the FPCA model have been developed. 
    These variants often shift the focus away from equalizing the reconstruction loss.
    In this paper, we identify a hidden convexity in the FPCA model 
    and introduce an algorithm for convex optimization via eigenvalue optimization. 
    Our approach achieves the desired fairness in reconstruction loss without sacrificing performance. 
    As demonstrated in 
    real-world datasets, the proposed FPCA algorithm runs $8\times$ faster than the SDR-based algorithm, and 
    only at most 85\% slower than the standard PCA. 
\end{abstract}

\section{Introduction} \label{sec:intro}

\paragraph{Fairness in Machine Learning.}
Machine learning has revolutionized decision-making, 
but concerns about fairness persist due to various levels of bias
throughout the process.
Bias can arise from skewed data,  
such as non-representative samples or measurement errors \citep{Caton:2024, Mehrabi:2021}, 
as well as from algorithms that prioritize overall accuracy at the expense of fairness \citep{Chouldechova:2018, Hardt:2016}.
Addressing these biases is essential to prevent discriminatory outcomes
and ensure the integrity and reliability of the decision-making procedure.

\paragraph{PCA and Fair PCA.}
PCA is arguably the most prominent linear dimensionality reduction technique
in machine learning and data science~\citep{Hotelling:1933, Pearson:1901}. 
However, the standard PCA ignores disparities between subgroups
and inadvertently leads to biased or discriminatory representations,
which can have particularly harmful consequences in socially impactful applications. 
This challenge has prompted the development of various fair PCA models to address biases. 
Approaches to fair PCA generally fall into two categories:
(a)~equalizing distributions of dimensionality-reduced data and 
(b)~equalizing approximation errors introduced by dimensionality reduction across subgroups.

Approaches in the first category focus mostly on mitigating statistical inference of sensitive attributes in the projection.
For example, 
\cite{Olfat:2019} aims to reduce disparities
in group means and covariance of projected data 
using semidefinite programming.
\cite{Lee:2022} reduces the maximum mean discrepancy
between distributions of protected groups through an exact penalty method.
\cite{Kleindessener:2023} seeks to achieve statistical independence between the 
projected data and its sensitive attributes, 
by mapping the data onto the null space of a vector involving those attributes.
Meanwhile, \cite{Lee:2023} introduces a ``null it out'' approach,
which nullifies the directions in which the sensitive attribute can be inferred,
using unfair directions including the mean difference and eigenvectors of the second 
moment difference, via a noisy power method.

In  the second category of approaches,
\cite{Samadi:2018} introduces a Fair PCA (FPCA) model aimed at equalizing {\em reconstruction loss}
across subgroups by minimizing the maximal reconstruction losses.
The proposed algorithm relies on semidefinite relaxation
and involves semidefinite programming followed by linear programming. It 
is computationally expensive and can operate ``{\em10 to 15 times}'' slower
than the standard PCA  \cite{Samadi:2018}.
Moreover, the semidefinite relaxation may introduce extra dimensions in the projection 
subspace in order to meet the fairness constraints. An upper bound on the extra dimensions is provided by~\cite{Tantipongpipat:2019}.
To address computational challenges arising in semidefinite relaxation,
\cite{Kamani:2022} introduces  a new framework
to minimize both overall reconstruction error and group fairness
in the Pareto-optimal sense by adaptive gradient descent.
In the same context of multi-objective optimization,
using the disparity between reconstruction errors as a fairness measure,
\cite{Pelegrina:2022} introduces a so-called strength Pareto evolutionary algorithm.
Subsequently,  
\cite{pelegrina:2024} proposes to cast the multi-objective problem as a single-objective 
optimization by optimally weighting the  objective functions, 
which is then solved by eigenvalue decompositions.
More recently, \cite{Xu:2024FPCA} introduces an Alternating Riemannian/Projected Gradient Descent Ascent (ARPGDA) algorithm for the general fair PCA problem.  
Building on this, \cite{Xu:2024} proposed a Riemannian Alternating Descent Ascent (RADA) framework for nonconvex-linear minimax problems on Riemannian manifolds. 


\paragraph{Contributions.}
The primary goal of this work is to revisit the
FPCA approach proposed by~\cite{Samadi:2018} 
and to present a novel algorithm that directly computes the FPCA
without relying on semidefinite relaxation and linear programming. 
Our contributions are threefold: 

\begin{itemize}
\setlength{\itemsep}{0pt}
\item 
We uncover a hidden convexity of the FPCA model by reformulating it as a minimization of a convex function 
over the joint numerical range of two matrices. 
This reformulation facilitates the development of an efficient solver and provides a geometric interpretation of the FPCA model.

\item 
We develop an efficient and reliable algorithm to solve the 
resulting  
convex optimization problem 
using univariate eigenvalue optimization.  
This approach directly yields 
the orthogonal basis $U$
for the projection subspace of the FPCA. 

\item 
We validate our method through extensive experiments on human-centric datasets, 
demonstrating that the new algorithm produces numerically accurate solutions, and meanwhile, gains up to $8\times$ 
speedup over the FPCA algorithm via semidefinite relaxation. 
Compared to standard PCA (without fairness constraints), 
the new algorithm is only at most an $85.81\%$ slowdown.
\end{itemize}

\paragraph{Organization.}
The rest of this paper is organized as follows:
\Cref{sec:pca} reviews the fundamentals of the standard PCA and its fairness issue.
\Cref{sect:fairpca} defines the fair PCA problem and 
uncovers a hidden convexity in it through the joint numerical range.
\Cref{sec:algorithm} details the eigenvalue optimization and implementation of our algorithm. 
\Cref{sec:experiment} presents numerical experiments comparing our algorithm with the original
fair PCA method by \cite{Samadi:2018} and the standard PCA. 
Finally, \Cref{sec:conclusion} provides concluding remarks.

\paragraph{Notations.}
We use standard notations in matrix analysis,
including  $\Tr(\cdot)$ for the trace of a matrix,
$\sigma_i(\cdot)$ for the $i$-th largest singular value, 
and  $\lambda_i(\cdot)$ for the $i$-th smallest eigenvalue.
The set of  orthogonal matrices is denoted by
\begin{equation}\label{eq:oo}
\mathbb O^{n \times r} := \left\{ U \in \mathbb{R}^{n \times r} : U^{T} U = I_r \right\}.
\end{equation}

\section{PCA and fairness issue} \label{sec:pca}

\paragraph{Principal Component Analysis.}
PCA is a fundamental technique
for dimensionality reduction. 
Let  $M \in \RR^{m \times n}$ be a data matrix,
where each row represents a sample with $n$ features, and assume that $M$ is centered, that is ${\bf 1}^TM=0$. 
The goal of PCA is to find  a projection basis 
$U \in \mathbb O^{n \times r}$ that reduces the feature dimension $n$ and best captures the variance in the data $M$. This optimal projection is found by 
minimizing the {\em reconstruction error}
\begin{equation} \label{eq:pca}
    \min_{U \in \mathbb O^{n \times r}} \norm{M -M UU^{T}}_{F}^{2}.
\end{equation}
It is well known \citep[pp.534-541]{Hastie:2009}  that the optimization problem~\eqref{eq:pca} 
is equivalent to the trace maximization 
\begin{equation}
\max_{U\in\mathbb O^{n \times r}} \Tr\left(U^TM^TMU\right),
\end{equation}
where the optimal projection basis is given by 
$U_{\!_M} \in\mathbb O^{n \times r}$ consisting of the orthogonal 
eigenvectors of the largest $r$ eigenvalues of the matrix $M^TM$.

\paragraph{Fairness issue.}
Real-world datasets often contain distinct subsets of data with unique characteristics such as gender, race, or other attributes.
Similar to \cite{Samadi:2018},
in this paper, we will focus on the case where the dataset consists of
two subgroups, $A$ and $B$,
namely, the entire dataset contains $m=m_1+m_2$ data points,
with $m_1$ in the subgroup $A$ and $m_2$ in subgroup B.
The data matrix $ M \in \mathbb{R}^{m \times n}$ can then be written as
\begin{equation}\label{eq:mmat}
M = \begin{bmatrix} A\\ B \end{bmatrix},
\end{equation}
where $A \in \mathbb{R}^{m_{1} \times n}$  and $B \in \mathbb{R}^{m_{2} \times n}$.
Fairness issues arise when the standard PCA is applied for analyzing the whole dataset $M$ to capture the maximum variance of $M$, as it may overlook disparities between subgroups and lead to an unfair solution where one subgroup is disproportionately affected.
See \Cref{fig:fairplot2} for an illustrative example.
Alternatively, applying the standard PCA separately to each subgroup
with two projection matrices
would neglect cross-group information and raise ethical concerns. 
Therefore, it is necessary  to strike a balance 
by finding a single projection subspace for the entire dataset while accounting for  
disparities between subgroups.

\begin{figure}[]
    \centering
    \includegraphics[width=0.8\linewidth]{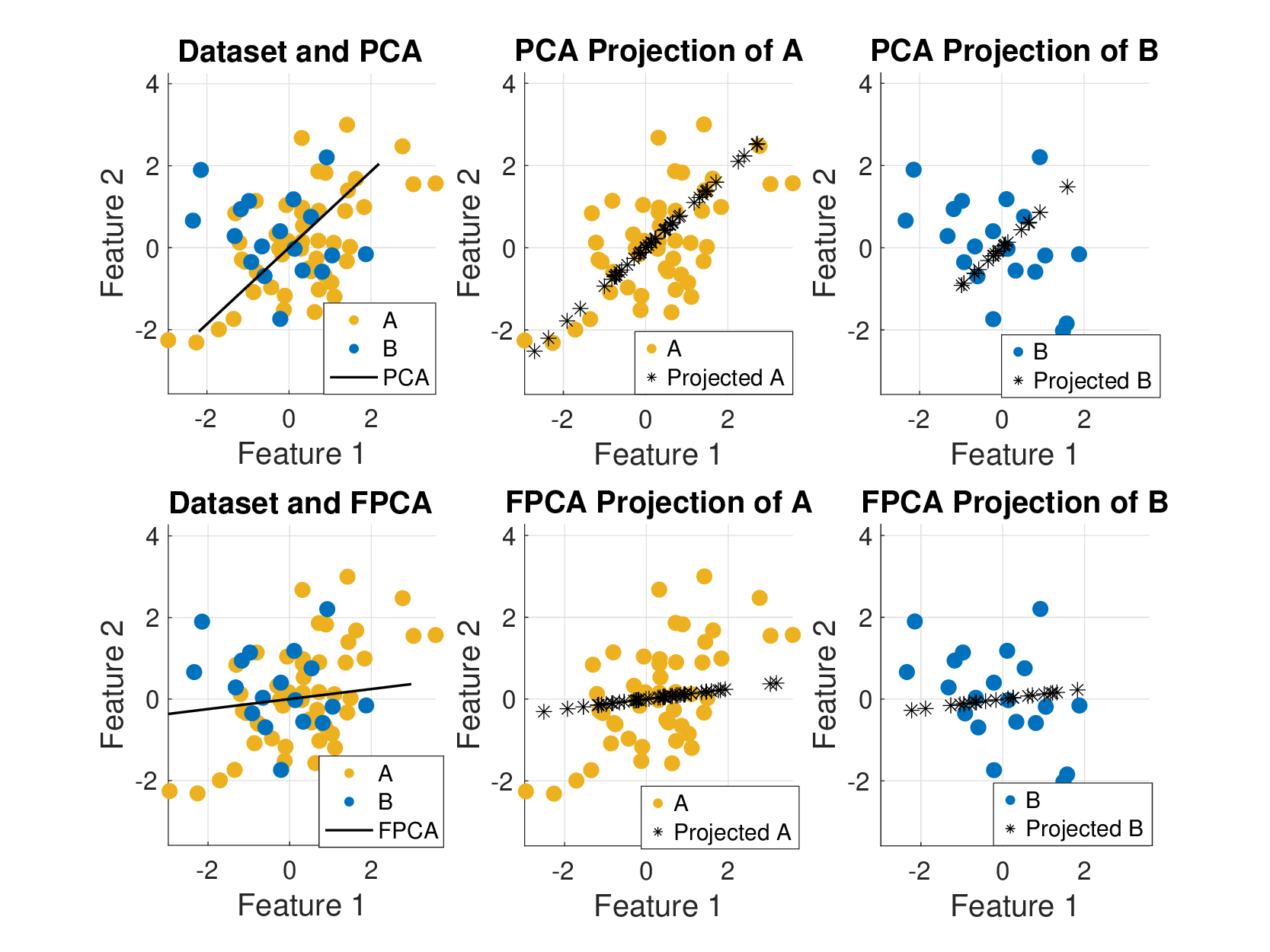}
    \caption{
    \underline{Top panel:} 
    the leading principal components of $M$ by the standard PCA, which captures the maximum variance of $A$, 
    at the expense of $B$, leading to an unfair projection.
    \underline{Bottom panel:}
    The FPCA  effectively reduces the imbalance in the variances captured. 
    }  
    \label{fig:fairplot2}
\end{figure}

\section{Fair PCA and hidden convexity} \label{sect:fairpca}
In this section, we revisit the fair PCA model 
and reveal a hidden convexity of the model. 

\subsection{Fair PCA model}

The major goal of fair PCA by \cite{Samadi:2018} is to 
find an optimal  projection scheme that achieves
equity in the (average) {\em reconstruction loss} between individual subgroups. 

\begin{definition} \label{def:loss}
    The (average) {\em reconstruction loss} for a given data matrix $D\in\RR^{p\times n}$
    by a  projection basis matrix $U\in\mathbb O^{n \times r}$ is defined as
    \begin{equation}\label{eq:loss} 
        \loss_{\!_D}(U)\! =\! \frac{1}{p}\Big(\|D \! - \! DUU^T\|_F^2 \!-\! \| D \! - \! D U_{\!_D} U_{\!_D}^{T} \|_F^2\Big),
    \end{equation}
    where $ {U}_{\!_D} \in \mathbb O^{n \times r}$ denotes the 
    PCA solution for $D$,
    that is $U_{\!_D}$ consists of orthogonal eigenvectors corresponding to the largest $r$ eigenvalues of $D^TD$.
\end{definition}
Recall that the 
solution $U_{\!_D}$ by PCA provides the optimal projection basis for the matrix $D$,
achieving the minimal reconstruction error.
Hence the $\loss_{\!_D}(U)$ measures how much worse a given projection basis $U$ is compared to 
the optimal solution $U_{\!_D}$.

\paragraph{Fair PCA model.}
By~\Cref{def:loss},
a projection by the basis matrix $U\in\mathbb O^{n \times r}$ is a fair projection
for two subgroups $A \in \RR^{m_{1} \times n}$ and $B \in \RR^{m_{2} \times n}$
if it ensures equity in their reconstruction losses that is
\begin{equation} \label{eq:fairobj}
 \loss_{{\!_A}}(U) = \loss_{{\!_B}}(U).
\end{equation}
In other words, the linear dimensionality reduction should represent the two subgroups $A$ and $B$ with equal fidelity. 
To achieve the fairness \eqref{eq:fairobj}, 
the following Fair PCA (henceforth FPCA) model is introduced
by~\cite{Samadi:2018}:
\begin{equation} \label{eq:minfdef}
\min_{U\in\mathbb O^{n \times r}}  \max\bigg\{ \loss{{\!_A}}(U),\ \loss_{{\!_B}}(U)\bigg\}.
\end{equation}
Intuitively, the model \eqref{eq:minfdef} minimizes the maximum reconstruction loss
to prevent a significant loss from disproportionately affecting any subgroup. 
See \Cref{fig:fairplot2} for an illustrative example.
Indeed, as shown in \Cref{thm:ustareq}, the solution of the minimax problem \eqref{eq:minfdef} 
guarantees fairness by achieving equal reconstruction loss for subgroups $A$ and $B$ as defined in~\eqref{eq:fairobj}. 

\subsection{Fair PCA as trace optimization}
The following lemma shows that the loss function~\eqref{eq:loss} can be written as a trace function. Subsequently, we can recast the FPCA model~\eqref{eq:minfdef} as 
minimizing the maximum of two traces.

\begin{lemma} \label{lemma:traceformulation}
Let $D\in\RR^{p\times n}$, $U\in \mathbb{O}^{n \times r}$, and $\loss_{\!_D}(U)$ be as defined in~\eqref{eq:loss}. 
We have 
\begin{equation}\label{eq:lossd}
\loss_{\!_D}(U) \equiv  \Tr(U^{T} H_{\!_D} U),
\end{equation}
where  $H_{\!_D}\in  \RR^{n \times n}$ is  a symmetric matrix given by
\begin{equation} \label{eq:matd}
    H_{\!_D}  \equiv
    \frac{1}{p} \left( \left[\frac{1}{r}  \sum\limits_{i = 1}^{r} \sigma^2_{i}(D)\right]
    \cdot I_{n} - D^{T} D \right),
\end{equation}
where $\sigma_i(D)$ denotes the $i$-th largest singular value of $D$.
\end{lemma}

\begin{proof} 
 By a straightforward derivation, we obtain the identity 
    \begin{align*}
        \norm{D(I_n -  U U^{T})}_{F}^2   = 
        \Tr\Big(D(I_n-UU^T)D^T\Big)=
        &\Tr(D^{T} D) -\Tr(U^{T} D^{T} D U),
    \end{align*}
    where we used $\|M\|_F^2\equiv \Tr(MM^T)$ and $U^TU=I_r$ in the first equation, 
    and $\Tr(AB)=\Tr(BA)$ in the second equation.
    Consequently, 
    \begin{align*}
    \loss_{D} (U) 
    \equiv \frac{1}{p} 
    \bigg(\norm{D (I_n- U U^{T})}_{F}^{2} - \norm{D (I_n-  U_{\!_D} U_{\!_D}^{T})}_{F}^{2}\bigg)
    =\frac{1}{p} \bigg(\Tr(U_{\!_D}^{T} D^{T} D U_{\!_D}) - \Tr(U^{T} D^{T} D U)\bigg).
    \end{align*}
    Since $U_{\!_D}$ contains the eigenvectors for the $r$ largest eigenvalues of $D^TD$, 
    or equivalently, the right singular vectors for the leading $r$ singular values of $D$,
    we have
    \[
    \Tr(U_{\!_D}^{T} D^{T} D U_{\!_D}) \equiv \sum_{i=1}^{r} \sigma_{i}(D)^2 
    = \Tr\left(U^T \left[ \frac{1}{r} \sum_{i=1}^{r} \sigma_{i}(D)^2\cdot  I_n\right] U\right).
    \]
    Combining the two equations from above, we proved~(3.4).
\end{proof}

By~\Cref{lemma:traceformulation},
we can reformulate the FPCA model~\eqref{eq:minfdef} 
as the minimax problem of two traces: 
\begin{equation} \label{eq:trminmax} 
\min_{U \in \mathbb O^{n \times r}}\max \bigg\{\Tr(U^{T} H_{\!_A} U),\ \Tr(U^{T} H_{\!_B} U)\bigg\},
\end{equation}
where 
\begin{equation} \label{eq:HAHBdef}
    H_{\!_A} =  \frac{1}{m_{1}} \left(
    \left[\frac{1}{r}  \sum\limits_{i = 1}^{r} \sigma^2_{i}(A)\right] \cdot I_{n} - A^{T} A \right)  \quad \mbox{and} \quad 
    H_{\!_B} =  \frac{1}{m_{2}} \left( \left[\frac{1}{r}  \sum\limits_{i = 1}^{r} \sigma^2_{i}(B)\right]  \cdot I_{n} - B^{T} B \right).
\end{equation}

An immediate benefit of formulating the FPCA model \eqref{eq:minfdef} to the minimax of two traces ~\eqref{eq:trminmax}
is a much simpler proof than in \cite{Samadi:2018} for the fairness condition~\eqref{eq:fairobj} of the optimal solution $U_*$. 

\begin{theorem} \label{thm:ustareq}
    The  solution $U_*\in \mathbb O^{n \times r}$ of the minimax problem~\eqref{eq:trminmax} satisfies
\begin{equation} \label{eq:ustareq}
    \Tr(U_*^TH_{\!_A}U_*) =  \Tr(U_*^TH_{\!_B}U_*).
\end{equation}
\end{theorem}

\begin{proof}
By contradiction,  
assume that the equality in~\eqref{eq:ustareq} does not hold.
Without loss of generality, suppose that 
\begin{equation}\label{eq:trfact1}
\Tr(U_*^TH_{\!_A}U_*) >  \Tr(U_*^TH_{\!_B}U_*).
\end{equation}
Since $\Tr(\cdot)$ is a continuous function, the inequality~\eqref{eq:trfact1} implies 
for all $U\in \mathbb O^{n \times r}$ 
sufficiently close to $U_*$,
\begin{equation*}
 \Tr(U^{T} H_{\!_{A}} U) > \Tr(U^{T} H_{\!_{B}} U).
\end{equation*}
Consequently, for all $U\in \mathbb O^{n \times r}$ that is sufficiently close to $U_*$, we have
\begin{equation*} 
\Tr(U^{T} H_{\!_{A}} U) \equiv 
\max \Big\{ \Tr(U^{T} H_{\!_{A}} U), \Tr(U^{T} H_{\!_{B}} U)\Big\} \geq \max \Big\{ \Tr(U_{*}^{T} H_{\!_{A}} U_{*}), \Tr(U_{*}^{T} H_{\!_{B}} U_{*})\Big\} \equiv \Tr(U_{*}^{T} H_{\!_{A}} U_{*}),
\end{equation*}
where the inequality is  due to $U_{*}$ being a minimal solution of the 
FPCA \eqref{eq:trminmax}, that is
$U_*$ is a local minimum of the trace minimization 
\begin{equation}\label{eq:a_tracemin}
\min_{U\in \mathbb O^{n \times r}} \Tr(U^{T} H_{\!_{A}} U).
\end{equation}
Recalling that  any local minimizer of the trace minimization~\eqref{eq:a_tracemin} 
must be a global minimizer 
\citep{Striko:1995}, 
$U_*$ must be a global minimizer of~\eqref{eq:a_tracemin} as well.
It then follows that 
\begin{equation} \label{eq:localglobalmin}
\Tr(U_*^{T} H_{\!_{A}} U_*) 
=
\min\limits_{U  \in \mathbb O^{n \times r}}\Tr(U^{T} H_{\!_{A}} U) 
= 
0,
\end{equation}
where the second equation is due to the fact
that 
$\Tr(U^{T} H_{\!_{A}} U) \equiv \loss_{\! A}(U)\geq 0$
with equality holding at the PCA solution $U_A\in\OO^{n\times r}$ for the matrix $A$. 
The equality~\eqref{eq:localglobalmin} and the inequality \eqref{eq:trfact1} lead to 
$0  >  \Tr(U_*^TH_{\!_B}U_*) \equiv \loss_{\! B} (U_*) \geq 0$,
which is a contradiction. 
\end{proof}

\subsection{The SDR-based algorithm} \label{sec:FPCAviaSDR}

The FPCA~\eqref{eq:trminmax} involves the minimax of two traces over 
the Stiefel manifold $\mathbb O^{n \times r}$.
It is a non-convex optimization.
To solve this problem,
\cite{Samadi:2018} developed a semidefinite relaxation (SDR) approach.
The key ideas behind their algorithm can be summarized as follows.

By the identity $\Tr(AB) = \Tr(BA)$, the minimax problem~\eqref{eq:trminmax} 
can be written in terms of the projection matrix $P=UU^T$ as 
\begin{equation}\label{eq:traceprojection}
\min_{P=UU^T,\ U \in \mathbb O^{n \times r}}\!\!\!\max \Big\{\Tr(H_{\!_A} P),\ \Tr( H_{\!_B} P)\Big\}.
\end{equation}
Note that the objective function of the minimization is convex in $P$.
The problem is then transformed into a convex optimization by 
{\em relaxing} feasible set $\{P=UU^T\colon\ U \in \mathbb O^{n \times r}\}$  to 
\begin{equation}\label{eq:lprelax}
\{P\in\RR^{n\times n}\ \colon\  \Tr(P)\leq r,\ 0\preceq P\preceq I\}.
\end{equation}
This relaxation leads to a semidefinite programming (SDP) problem, 
which can be solved using existing techniques.
Once $P$ is computed, a linear programming (LP) step is applied to {\em correct}
the rank of $P$.

This SDP and LP-based approach  
has two major drawbacks.
First, the algorithm produces an approximate projection 
$\widehat P\in\RR^{n\times n}$, instead of an orthogonal basis $U\in\mathbb O^{n \times r}$.
Due to SDR and computational error, the resulting 
$\widehat P$ may
{fail to recover the orthogonal projection} $UU^T$ of rank $r$.
Secondly,  this approach is 
{expensive},
with a theoretical runtime of  $O \left({n^3}/{\text{tol}^2} \right)$
where $\text{tol}$ is the error tolerance of the SDP and LP.  As reported in~\cite{Samadi:2018}, the runtime is
``\textit{at most $10$ to $15$ times}'' slower than the standard PCA.

\subsection{Hidden convexity}\label{sec:hidden}

In this section, we will uncover a hidden convexity of the FPCA model~\eqref{eq:trminmax}
by a change of variables. 
This allows us to visualize the optimization~\eqref{eq:trminmax},
as well as  reformulate it as a convex optimization problem.

\paragraph{Joint numerical range.}
A $r$-th joint numerical range of a pair of 
symmetric matrices $S,T \in \RR^{n \times n}$ is defined as: 
\begin{equation*}
\mathcal{W}_{r}(S, T) = 
\left\{ 
\begin{bmatrix} \Tr(U^{T} S U) \\ \Tr(U^{T} T U) \end{bmatrix}\in\mathbb R^2 
: U \in \mathbb O^{n \times r}
\right\}.
\end{equation*}
The set 
$\mathcal{W}_{r}(S, T)$ is a bounded and closed 
subset of  $\RR^{2}$, since $\mathbb O^{n \times r}$ is closed and $\Tr(\cdot)$
is a continuous function \citep{Li:1994}. 
It is also known that $\mathcal{W}_{r}(S, T)$ is a convex subset of $\RR^{2}$ 
when the size of the matrices $n > 2$~\citep{Au:1984}.

\paragraph{Generating the joint numerical range.}
For visualization, we can generate the convex joint numerical range $\mathcal{W}_{r} (S, T)$  
by sampling its boundary points. In the following, we show that 
a boundary point can be obtained by solving a symmetric eigenvalue problem. 
This method extends the existing approach for computing the 
numerical range of a square matrix (see, e.g., 
\cite{Johnson:1978}).

Given a search direction $v:=[\cos(\theta),\sin(\theta)]^T \in\RR^2$ with $\theta\in(0,2\pi]$,
the boundary point $y_{\theta}$ of  $\mathcal{W}_{r} (S, T)$ with an outer normal vector $v$
can be obtained by solving the optimization problem 
\begin{equation}\label{eq:nrbdp}
y_{\theta} =\mathop{\arg\max}_{y\in  \mathcal{W}_{r} (S, T)} \, v^Ty,
\end{equation}
By parameterizing $y\in \mathcal{W}_{r} (S, T)$ 
as $y=[\Tr(U^TSU),\Tr(U^TTU)]^T$ for $U\in\OO^{n\times r}$, the maximization~\eqref{eq:nrbdp} is rewritten as
\[
\max_{y\in  \mathcal{W}_{r} (S, T)}  v^Ty = 
\max_{U\in\OO^{n\times r}}  \Tr\Big(U^T[\cos(\theta)\cdot S + \sin(\theta)\cdot T]U\Big).
\]
By Ky Fan's trace optimization principle, the solution $U_{\theta}$ to the above trace maximization  is given by the orthogonal eigenvectors corresponding to the largest $r$ eigenvalues of the matrix 
\begin{equation*}
B(\theta) := \cos(\theta) S + \sin(\theta) T.
\end{equation*}
Thus,  the boundary point along the direction $v$ is given by
\begin{equation*}
y_\theta= [\text{Tr}(U_\theta^{T} S U_\theta), \text{Tr}(U_\theta^{T} T U_\theta)]^T.
\end{equation*}
By searching in different directions, such as using equally spaced angles  $\theta_j$  in $(0,2\pi]$, 
we can sample a finite number of boundary points of $\mathcal{W}_{r} (S, T)$. 
The convex hull of these boundary points generates an approximate joint numerical range, as shown in Figure~\ref{fig:geom_interpolation}. 

\begin{figure}[]
    \centering
    \includegraphics[width=0.5\linewidth]{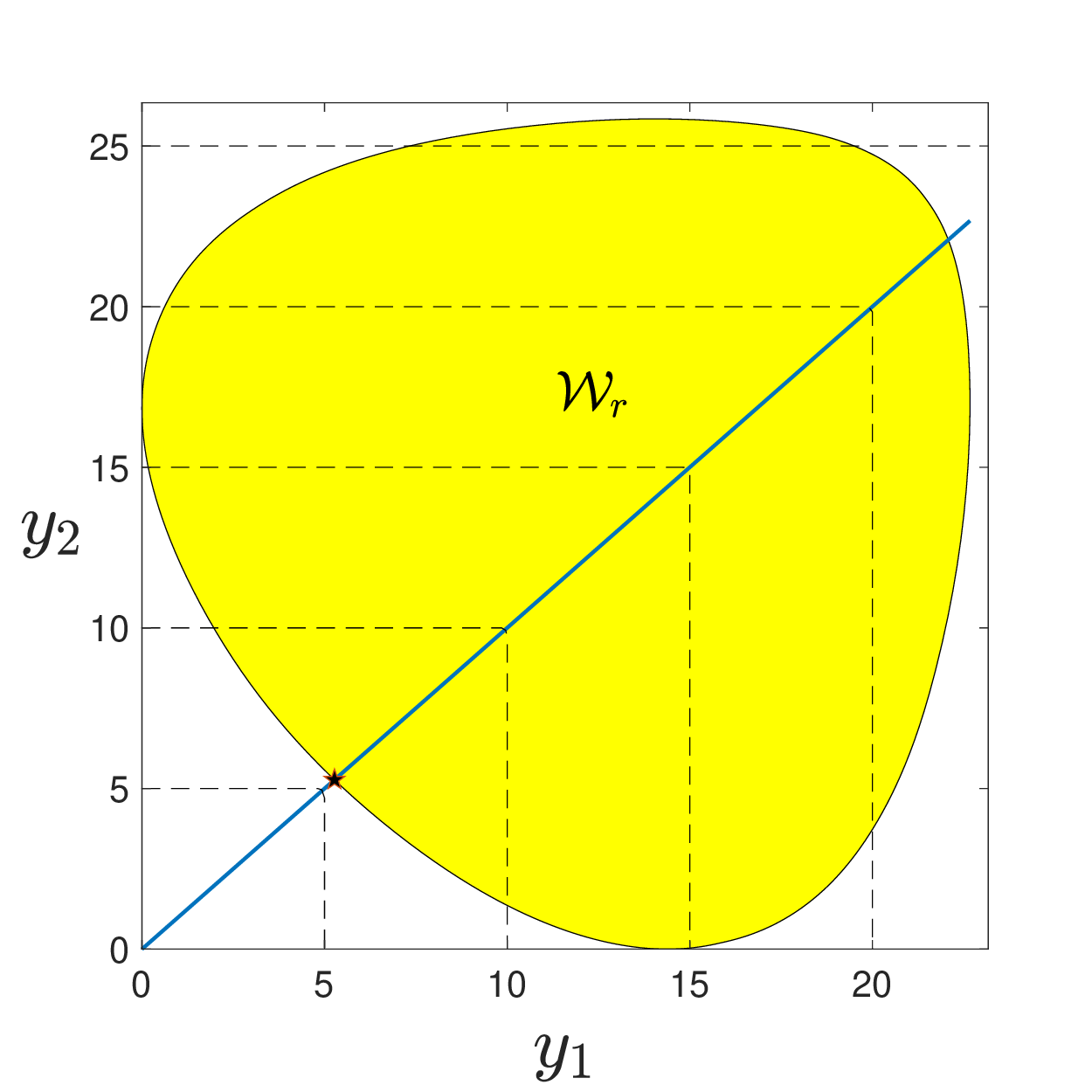}
    \caption{Geometric illustration of  the FPCA \eqref{eq:trminmax}:
    The yellow region is the joint numerical range $\mathcal W_{r}(H_{\!_A},H_{\!_B})$.
    Each dashed line is a contour of the `max' function, i.e., solution of $\max\{y_1,y_2\} = c$ 
    for a given constant $c$.
    The solid blue line is with $y_1 = y_2$.
    The star marks the optimal solution $y_*$ of~\eqref{eq:trminmax}.
    }
    \label{fig:geom_interpolation}
\end{figure}

The overall procedure for generating the joint numerical range is summarized in \Cref{alg:jnr}. 

\begin{algorithm}[]
\caption{Generating the joint numerical range $\mathcal{W}_{r}(S, T)$}
\label{alg:jnr}
\begin{algorithmic}[1]

\REQUIRE{Symmetric matrices $S \in \RR^{n \times n}$, $T \in \RR^{n \times n}$,
dimension $r$, and number of angle samples $\ell$.}

\ENSURE{Approximate $\mathcal{W}_{r}(S, T)$
by the convex hull of the boundary points $\{y_1,\dots,y_\ell\}$.}


\STATE Set step size for search angles in $(0,2\pi]$ as $h = 2\pi/\ell$; 

\FOR{$j=1,2,\dots,\ell$}
    \STATE Set the search angle $\theta= jh$;
    
    \STATE Compute eigenvectors  $U_\theta$ corresponding to 
    the largest $r$ eigenvalues of $B(\theta) = \cos(\theta) S + \sin(\theta) T$;

    \STATE Compute the boundary point coordinates
    $y_j= [\text{Tr}(U_\theta^{T} S U_\theta), \text{Tr}(U_\theta^{T} T U_\theta)]^T$.
    
\ENDFOR

\STATE Return the convex hull of $\{y_1,\dots,y_\ell\}$
as an approximation of $\mathcal{W}_{r}(S, T)$.

\end{algorithmic} 
\end{algorithm}

\paragraph{Optimization over the joint numerical range.} 
Let  $y\in\mathbb R^2$ be given by 
\begin{equation} \label{eq:y1y2}
y = 
\begin{bmatrix} y_1 \\ y_2  \end{bmatrix}
\equiv 
\begin{bmatrix} \Tr(U^{T} H_{\!_A} U) \\ \Tr(U^{T} H_{\!_B} U)  \end{bmatrix}.
\end{equation}
By a change of variables from $U$ to $y$, 
we can reformulate the FPCA model \eqref{eq:trminmax} 
as the following minimization problem over the joint numerical range:
\begin{equation} \label{eq:fairPCAdef4}
    \min_{y \in \mathcal{W}_{r}(H_{\!_A},H_{\!_B})} \max \{y_{1},y_{2}\}.
\end{equation}
Observe that the objective function $\max \{y_{1}, y_{2}\}$ is a convex function 
in $y \in \RR^{2}$  \citep[pp.72]{Boyd:2004}.
Moreover, the feasible set $\mathcal{W}_{r}(H_{\!_A}, H_{\!_B})$ 
is a convex set under the general assumption of $n > 2$.
Therefore, the optimization~\eqref{eq:fairPCAdef4} is 
a convex optimization problem when $n > 2$.\footnote{
The requirement $n>2$ is not essential,
as we can instead use the joint numerical range 
with complex orthogonal $U\in\mathbb C^{n\times r}$,
which is always convex 
.}

\paragraph{Geometric interpretation of FPCA.} 
The convex optimization problem~\eqref{eq:fairPCAdef4}
involves only two variables $y_1$ and $y_2$. 
Combining with the convexity of the joint numerical range $\mathcal W_r(H_{\!_A},H_{\!_B})$, it 
allows us to visualize the solution of the FPCA. 
Figure~\ref{fig:geom_interpolation} depicts 
$\mathcal W_r(H_{\!_A},H_{\!_B})$ for a random example,
along with the contours of the objective function  $\max\{y_1,y_2\}$.
The figure shows that the optimal solution $y_*$ of~\eqref{eq:fairPCAdef4} 
is achieved at the intersection of  $\mathcal W_r(H_{\!_A},H_{\!_B})$
and the diagonal of Cartesian coordinates (i.e. $y_1=y_2$),
visually confirming the fairness of the solution of the FPCA as described
in~\Cref{thm:ustareq}.

This visualization also provides a geometric interpretation of~\Cref{thm:ustareq},
explaining why a fair solution is always achievable.
For the matrices $(H_{\!_A},H_{\!_B})$ of the FPCA model~\eqref{eq:trminmax}, the joint numerical range $\mathcal W_r(H_{\!_A},H_{\!_B})$ 
always lies in the first quadrant of the coordinate plane, and  intersects with both the vertical and horizontal axes
(as shown in Figure~\ref{fig:geom_interpolation}).
These follow from~\eqref{eq:lossd} and~\eqref{eq:loss}, which imply  that 
$y_1 = \Tr(U^TH_{\!_A}U) = \loss_{\!_A}(U)\geq 0$  for all $U\in\mathbb O^{n \times r}$
with equality holding at  the PCA solution $U=U_{\!_A}$ for the matrix $A$,
and 
$y_2 = \Tr(U^TH_{\!_B}U) = \loss_{\!_B}(U) \geq 0$  for all $U\in\mathbb O^{n \times r}$
with equality holding at $U=U_{\!_B}$.
Consequently, the intersection of  $\mathcal W_r(H_{\!_A},H_{\!_B})$ 
with the diagonal line always exists
and corresponds to the solution of the FPCA \eqref{eq:trminmax}.

\section{Algorithm} \label{sec:algorithm}
In this section,  we present an alternative algorithm 
for the FPCA \eqref{eq:trminmax} via the convex optimization~\eqref{eq:fairPCAdef4}
and a univariate eigenvalue optimization.

\subsection{Eigenvalue optimization} 
By the geometric interpolation as shown in~Figure~\ref{fig:geom_interpolation}, 
we can search along the boundary of 
the joint numerical range $\mathcal W_r(H_{\!_A},H_{\!_B})$ for the optimal $y_*$ of  \eqref{eq:fairPCAdef4}.
The following theorem shows that the optimal $y_*$ can be obtained via eigenvalue optimization of the 
following symmetric matrix function over $t\in[0,1]$:
\begin{equation}\label{eq:ht}
H(t)\equiv t \cdot H_{\!_A} + (1-t) \cdot H_{\!_B}.
\end{equation}

\begin{theorem}\label{thm:eigopt}
The optimal solution of the FPCA~\eqref{eq:fairPCAdef4} is given by 
\begin{equation}\label{eq:ystar}
y_* =  \begin{bmatrix} \Tr(U_*^TH_{\!_A}U_*)\\ \Tr(U_*^TH_{\!_B}U_*)\end{bmatrix}
\end{equation}
where  $U_*\in\mathbb O^{n \times r}$ is a basis matrix 
for the eigenspace corresponding to the $r$ smallest eigenvalues
of $H(t_*)$, and $t_*$ is the solution of the following eigenvalue optimization 
\begin{equation} \label{eq:maxphi}
    \max_{t \in [0,1]} \left\{ \phi(t) \equiv \sum_{i = 1}^{r} \lambda_i(H(t)) \right\},
\end{equation}

\end{theorem}

\begin{proof} 
Let us first prove the theorem for the general cases where the matrix size $n>2$.
We begin by parameterizing the  FPCA \eqref{eq:fairPCAdef4} model: 
\begin{align}
\min_{y \in \mathcal{W}_{r}(H_{\!_A},H_{\!_B})}\ \max\{y_1,y_2\} 
&=  \min_{y \in \mathcal{W}_{r}(H_{\!_A},H_{\!_B})}\ \max_{t \in [0,1]} 
    \big[   t \cdot y_1 + (1 - t) \cdot y_2 \big] \notag\\
&=   \max_{t \in [0,1]}\ \min_{y \in  \mathcal{W}_{r}(H_{\!_A},H_{\!_B})}  
    \big[   t \cdot y_1 + (1 - t) \cdot y_2 \big], \label{eq:minmax}
\end{align}
where the first equality is established by a straightforward verification,
and the second equality is from a generalized von Neumann's minimax theorem 
\citep{Sion:1958}.
We can use this theorem because the objective function is affine in $y$ and $t$ and 
both feasible sets, $\mathcal{W}_{r}(H_{\!_A},H_{\!_B}) $ and $[0,1]$, are convex.

We observe that the inner minimization in~\eqref{eq:minmax}
has a closed-form solution as derived in the following:
\begin{align}
    \min_{y \in \mathcal{W}_{r}(H_{\!_A},H_{\!_B})}   \big[   t \cdot y_1 + (1 - t) \cdot y_2 \big] 
  &= \min_{U\in\mathbb O^{n \times r}} \big[t \cdot  \Tr(U^TH_{\!_A}U) + (1-t)\cdot   \Tr(U^TH_{\!_B}U)\big] \notag\\
  &= \min_{U\in\mathbb O^{n \times r}} \Tr(U^TH(t)U)  = \sum_{i=1}^{r} \lambda_i(H(t)), \label{eq:sumeig}
\end{align}
where 
the first equation is by a parameterization of
$y =\begin{bmatrix} \Tr(U^TH_{\!_A}U), \Tr(U^TH_{\!_B}U)  \end{bmatrix}^T$
for $U\in\mathbb O^{n \times r}$,
the second equation is by the definition of $H(t)$ in~(4.1),
and  the last equation is by Ky Fan's eigenvalue minimization principle 
\citep{KyFan:1949}, 
which implies the minimal trace is given by the sum of the $r$ smallest eigenvalues 
of $H(t)$.
Hence, plugging the closed-form solution~\eqref{eq:sumeig} into 
the inner minimization of~\eqref{eq:minmax}, 
we write the 
FPCA model \eqref{eq:fairPCAdef4}
as an eigenvalue optimization problem:
\begin{equation}\label{eq:eigopt}
  \min_{y \in \mathcal{W}_{r}(H_{\!_A},H_{\!_B})}\ \max\{y_1,y_2\} 
  =\max_{t\in[0,1]} \phi(t),
\end{equation}
where $\phi(t)= \sum\limits_{i=1}^{r} \lambda_i(H(t))$.

Now, we consider the relation between the 
solution  $y_*$ and $t_*$ of the two optimization 
problems in~\eqref{eq:eigopt}.
It follows from~\eqref{eq:eigopt} that
\[
   \phi(t_*)=\max\{y_{*1},y_{*2}\} \geq  t_*y_{*1}+(1-t_*)y_{*2}  \geq \phi(t_*),
\]
where the first inequality is due to $t_*\in[0,1]$ 
and the second  is due to~\eqref{eq:sumeig} with a fixed $t=t_*$.
Since equalities must hold in the equation above, we have 
\[
 t_*y_{*1}+(1-t_*)y_{*2}  \equiv \phi(t_*).
\]
By the expression of 
$y_*$~(4.2),
this is equivalent to 
\[
\Tr(U^T_*H(t_*)U_*) = \sum_{i=1}^r \lambda_i(H(t_*)),
\]
which, according to Ky Fan's eigenvalue minimization principle 
\citep{KyFan:1949}
implies  $U_*$ must be an eigenbasis for the $r$ smallest eigenvalues of $H(t_*)$.

The rest of the proof is to address the special cases with $n=1$ and $2$.
First for the cases $n=r=1$ and $n=r=2$, 
it follows directly from  the definitions 
of  $H_{\!_A}$ and  $H_{\!_B}$ in~(3.7) that 
\[
\Tr(U^T H_{\!_A} U) \equiv  \Tr(U^T H_{\!_B} U) \equiv 0,
\]
for all  $U\in\OO^{n\times r}$.
This implies $\mathcal W_r(H_{\!_A},H_{\!_B})=\{0\}$, a convex set, 
so the above proof for $n>2$ still holds.

It remains to consider the case where $n=2$ and $r=1$.
In this case,  the joint numerical range  $\mathcal W_r(H_{\!_A},H_{\!_B})$  must 
form  a {\em general ellipse} (i.e.,  either an ellipse, a circle, a line segment, or a point)
\citep{Brickman:1961}.
Therefore,   $\mathcal W_1(H_{\!_A},H_{\!_B})$ consists of the boundary points of
its convex hull $\mbox{Conv}(\mathcal W_1(H_{\!_A},H_{\!_B}))$,
the smallest convex set that contains $\mathcal W_1(H_{\!_A},H_{\!_B})$.
Consequently,  the optimization problem~\eqref{eq:traceprojection} can be solved over 
this convex hull as
\begin{equation}\label{eq:minconv}
    \min_{y \in \mathcal{W}_{1}(H_{\!_A},H_{\!_B})} \max \{y_{1},y_{2}\}
    = 
    \min_{y \in \mbox{Conv}(\mathcal{W}_{1} (H_{\!_A},H_{\!_B}))} \max \{y_{1},y_{2}\},
\end{equation}
where we used the fact that $ g(y):=\max \{y_{1},y_{2}\} $ has no stationary point, 
so its minimizer must occur on the boundary of the feasible set.

It is well-known that the convex hull of $\mathcal{W}_{1} (H_{\!_A},H_{\!_B})$ is
exactly its complex analogue, as defined by
\begin{equation}\label{eq:cjnr}
\mathcal{W}_{1}^{\CC}( H_{\!_A},H_{\!_B})
:= 
\left\{ 
\begin{bmatrix} \Tr(U^{H} H_{\!_A} U) \\ \Tr(U^{H} H_{\!_B}  U) \end{bmatrix}
: U \in \CC^{n \times 1}, U^{H} U = 1
\right\},
\end{equation}
where $\cdot^H$ denotes conjugate transpose, and the superscript $\CC$ in  $\mathcal{W}_{1}^{\CC}$ 
is to distinguish it from the previous definition of joint numerical range $\mathcal W_r$ in \Cref{sec:hidden}
where $U$ is a real matrix; see \cite{Brickman:1961}.
Then we can write~\eqref{eq:minconv} as
\begin{equation}\label{eq:minconv2}
    \min_{y \in \mathcal{W}_{1}(H_{\!_A},H_{\!_B})} \max \{y_{1},y_{2}\}
    = 
    \min_{y \in \mathcal{W}_{1}^{\CC}( H_{\!_A},H_{\!_B})}
    \max \{y_{1},y_{2}\}.
\end{equation}
Since $\mathcal{W}_{1}^{\CC}( H_{\!_A},H_{\!_B})$ is a convex set,
we can apply the same proof used for~\eqref{eq:eigopt} 
to the optimization in~\eqref{eq:cjnr}
to establish the eigenvalue optimization problem~\eqref{eq:maxphi}.
This completes the proof for the case of $n=1, 2$.
\end{proof}

The following result shows that the eigenvalue function $\phi(t)$ is concave in
$t \in [0, 1]$, see Figure~\ref{fig:sample_phit} for an illustration.
As a result, the maximizer of $\phi(t)$ over $[0, 1]$ 
can be efficiently found using classical methods such as the golden section search.

\begin{lemma} \label{lemma:phiproperty}
The function $\phi(t)$ is continuous, piecewise smooth, and concave
over the interval $[0,1]$.
\end{lemma}
\begin{proof}
The piecewise smoothness of the eigenvalue function $\phi(t)$ follows directly from 
a  classical result in eigenvalue perturbation analysis, 
which states that the $i$-th eigenvalue $\lambda_i(H(t))$ of a symmetric matrix $H(t)$ 
is a piecewise analytic function of the entries of $H(t)$;
see, e.g., 
\citet[Chap.1]{Rellich:1969}.

The concavity of the eigenvalue function $\phi(t)$ is also well-known in convex analysis.
In our case, it can be quickly verified by~\eqref{eq:sumeig}, which states that
\[
\phi(t) \equiv \min_{y \in \mathcal{W}_{r}(H_{\!_A},H_{\!_B})} \Big\{ \phi_y(t):=   t \cdot y_1 + (1 - t) \cdot y_2  \Big\},
\]
namely, $\phi(t)$ is the pointwise minimum of a set of functions $\{\phi_y(t)\}$.
Since all $\phi_y(t)$ are concave functions in $t$,
we have 
$\phi(t)$ must be concave in $t$ as well; see, e.g., \citet[Sec.3.2.3]{Boyd:2004}.
\end{proof}

\begin{figure}[] 
    \centering
    \includegraphics[width=0.6\linewidth]{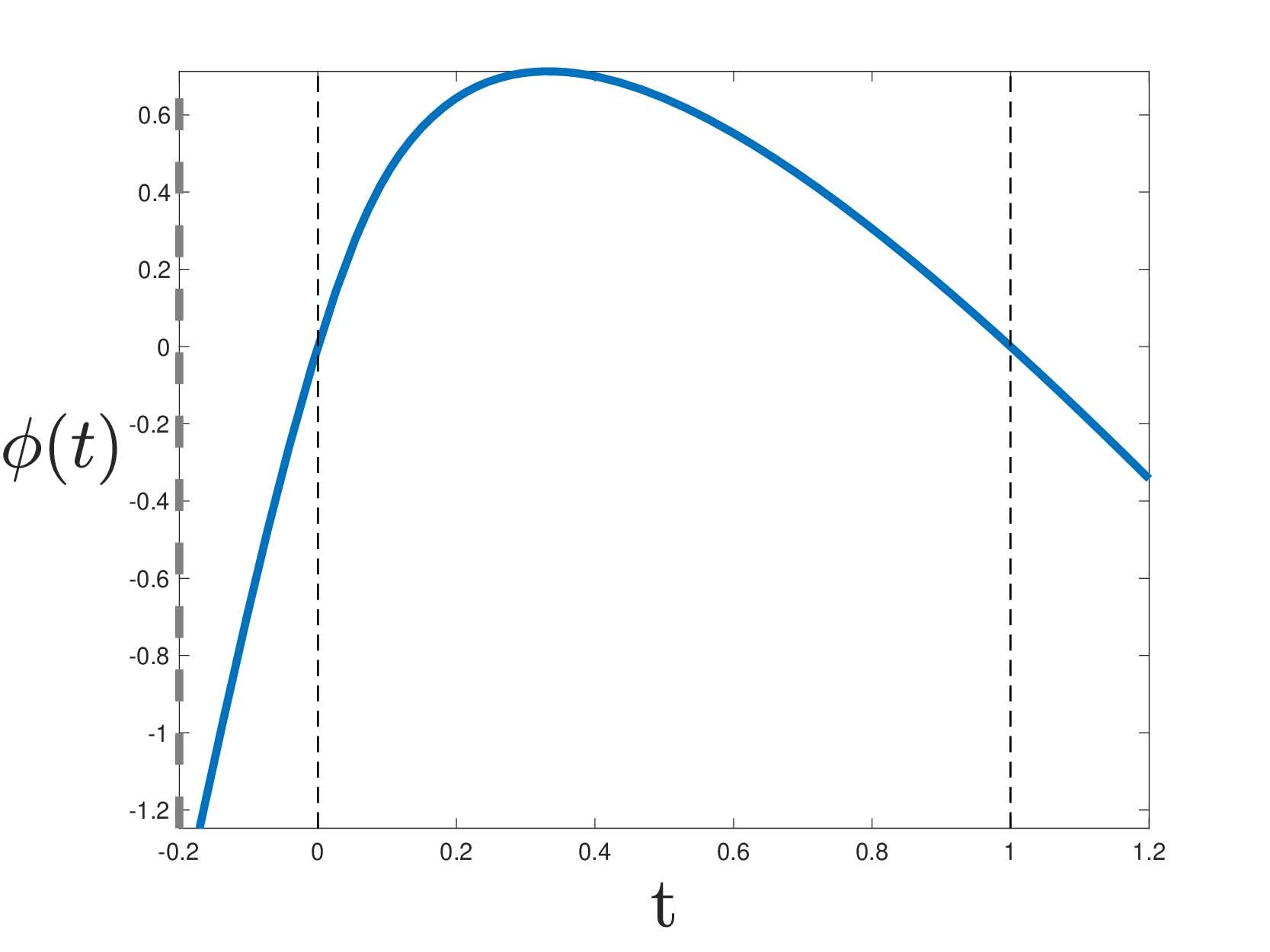}
    \caption{An illustration of the eigenvalue function $\phi(t)$ 
    }
    \label{fig:sample_phit}
\end{figure}

\subsection{Algorithm}
 \label{sec:alg}

\Cref{alg:eigsumopt} is an outline of the proposed algorithm for the FPCA \eqref{eq:trminmax} via  
the eigenvalue optimization. It is called EigOpt in short.

\begin{algorithm}[H]
\caption{EigOpt}
\label{alg:eigsumopt}
\begin{algorithmic}[1]

\REQUIRE{data $A \in \RR^{m_{1} \times n}$, $B \in \RR^{m_{2} \times n}$;
number of principal components $r < n$;
error tolerance tol.
} 
 
\ENSURE{the solution  $U_* \in \RR^{n \times r}$ of the FPCA \eqref{eq:trminmax}.}

\STATE \label{step:singular_values}
Compute the largest $r$ singular values 
of $A$ and $B$, respectively, 
that define the matrices 
$H_{\!_A}$ and $H_{\!_B}$ in ~\eqref{eq:HAHBdef} {\em implicitly}.

\STATE \label{step:find_t}
Compute $t_{*} \in [0,1]$ that maximizes $\phi(t)$ defined in \eqref{eq:maxphi}, 
using the absolute error tolerance $\text{tol}$.

\STATE \label{step:compute_eigenvectors}
Compute the eigenvectors $U_*$ corresponding to the smallest $r$ eigenvalues of $H(t_{*}) = t_* H_{\! A} + (1-t_*)H_{\! B}$.

\end{algorithmic} 
\end{algorithm}

A few remarks are as follows.
\begin{enumerate}
    \item 
    In step~\ref{step:find_t}~and~\ref{step:compute_eigenvectors},
    we need to compute the smallest $r$ eigenvalues of the symmetric matrix
    $H(t) = tH_{\!_A}+(1-t)H_{\!_B}$ for a given $t\in[0,1]$.
    If the matrix $H(t)$ is of a moderate size, 
    then we construct $H(t)$ explicitly, compute all its eigenvalues, 
    and select  the $r$ smallest ones.
    

When the size of the matrix $H(t)$ is large,
we can apply a Krylov subspace method to find its smallest $r$ eigenvalues.
In this case, we only need to supply the matrix-vector multiplications 
$u=H(t)\cdot v$  without explicitly formulating $H(t)$. 
We note that 
by the definitions of $H_{\!_A}$ and $H_{\!_B}$ in
\eqref{eq:HAHBdef}, 
\begin{equation} \label{eq:Ht_def2}
        H(t) =  t \cdot H_{\!_A} + (1 - t) \cdot H_{\!_B} 
        = \Big( t\cdot\gamma_{\!_A} + (1-t)\cdot\gamma_{\!_B}\Big) \cdot I_{n} 
        - \left(\frac{t}{m_{1}} \cdot A^{T} A + \frac{1 - t}{m_{2}} \cdot B^{T} B \right),
\end{equation}
where $\gamma_{\!_A}:= \frac{1}{m_1r}\sum_{i=1}^{r} \sigma_i^2(A)$
and $\gamma_{\!_B}:= \frac{1}{m_2r}\sum_{i=1}^{r} \sigma_i^2(B)$.
Consequently, the product $u=H(t)\cdot v$ can be evaluated in four matrix-vector multiplications with the matrices
$A$, $B$, $A^T$, and $B^T$. 

    \item 
    For the eigenvalue optimization in step~\ref{step:find_t}, 
    we can apply Brent's method~\citep[Sec.5.4]{Brent:2013},
    which combines the golden-section search with parabolic interpolation.
    Brent's method is {\em derivative-free},  {\em globally convergent}, 
    and {\em locally superlinearly convergent}~\citep[Sec.5.4]{Brent:2013}. 
    It is ideal for our eigenvalue optimization task and is available as MATLAB's built-in function~\texttt{fminbd}. Therefore, step~\ref{step:find_t} can be conveniently implemented 
    with just three lines of MATLAB code: 
    
\begin{lstlisting}[style=MATLAB]
H = @(t) t*HA + (1-t)*HB; 
phi = @(t) sum(eigs(H(t), r, 'smallestreal'));
t_star = fminbnd(@(t) -phi(t), 0, 1);
\end{lstlisting}

If necessary, one can specify the tolerance parameters for \texttt{eigs} and \texttt{fminbnd}.

    \item 
    In step~\ref{step:compute_eigenvectors}, we assume
    $\lambda_{r+1}(H(t_*))>\lambda_r(H(t_*))$, i.e., a gap exists between the eigenvalues. Therefore the eigenspace of the $r$ smallest eigenvalues of $H(t_*)$ is unique. 
Any orthogonal basis matrix $U_*\in\OO^{n \times r}$
can be used  to define the $y_*$ by~\eqref{eq:ystar}. 
The choice of the orthogonal basis $U_*$ is irrelevant,
since $\Tr(U^TH_{\!_A}U)$ and $\Tr(U^TH_{\!_B}U)$ 
are invariant under a right multiplication of $U$ by any $Q\in\OO^{n\times r}$.

The multiple eigenvalue for $\lambda_r$ is rare in practice and was not observed in our experiments. 
If we encounter the case  $\lambda_{r}(H(t_*)) = \lambda_{r+1}(H(t_*))$, 
the $r$-th eigenvalue may have multiple linearly independent eigenvectors,
and all of them belong to the eigenspace.
In this case,  we need to select $r$ particular eigenvectors to construct $U_*\in\OO^{n\times r}$.
However, a specific selection of $r$ eigenvectors,
denoted by  $\widehat U$,
may not satisfy the fairness condition 
$\Tr(\widehat U^T H_{\!_A} \widehat U ) =  \Tr(\widehat U^T H_{\!_B} \widehat U)$.
Consequently, in the case of multiple eigenvalues, we need an additional postprocessing step described as follows. 
Assume that $H(t_*)$  has eigenvalues ordered as 
\[
\lambda_1\leq \dots \leq \lambda_p < 
\underbrace{\lambda_{p+1} = \dots \lambda_r=\lambda_{r+1}=\dots = \lambda_{p+q}}_{q \text{ times}} < \lambda_{p+1+1},
\]
where the $r$-th eigenvalue has multiplicity $q$.
We partition accordingly the eigenvectors of $H(t_*)$ as 
\begin{equation}\label{eq:u1u2}
[U_1,U_2]\in\OO^{n\times (p+q)}
\quad \text{with $U_1\in \OO^{n\times p}$ and  $U_2\in \OO^{n\times q}$},
\end{equation}
where $U_1$ corresponds to the first $p$ eigenvalues, 
and $U_2$ corresponds to the repeated eigenvalues.
Since the FPCA solution $U_*$ contains eigenvectors corresponding to 
the smallest $r$  eigenvalues of $H(t_*)$, 
we construct  $U_*$ as 
\begin{equation}\label{eq:usv}
U_* = [U_1, U_2 V]\in\OO^{n\times r}
\text{~~~~~ for some $V\in\OO^{q\times (r-p)}$}.
\end{equation}
Our goal is to find a particular $V$ such that 
the fairness condition is satisfied:
\begin{equation}\label{eq:rootfind0}
0 =  \Tr(U_*^T H_{\!_A} U_* ) -  \Tr(U_*^T H_{\!_B} U_*)  =  \Tr(U_*^T [H_{\!_A} - H_{\!_B}] U_*)
=  \gamma +  \Tr( V^T C V) ,
\end{equation}
where $\gamma = \Tr(U_1^T (H_{\!_A}- H_{\!_B} )U_1)$ 
is a constant and 
$C = U_2^T (H_{\!_A}- H_{\!_B}) U_2\in \RR^{q\times q}$
is the difference matrix $H_{\!_A}- H_{\!_B}$ projected onto 
the eigenspace of repeated eigenvalues (often of a small size).


Note that we only need a particular
$V\in\OO^{q\times (r-p)}$ that satisfes the condition~\eqref{eq:rootfind0}.
This solution can be conveniently found by a line search.
First, the maximum and minimum values of  the function 
\[ g(V):=\gamma+\Tr(V^TCV)\]
are  respectively  achieved by  
the eigenvectors $V_M\in\OO^{q\times (r-p)}$ 
corresponding to the $r-p$ largest eigenvalues of $C$,
and  $V_m\in\OO^{q\times (r-p)}$ 
corresponding to the $r-p$ smallest eigenvalues of $C$.
We have 
\begin{equation}\label{eq:gvm}
g(V_m)\equiv 
\gamma +  \Tr( V_m^T C V_m) \leq  0 \leq  \gamma +  \Tr( V_M^T C V_M)
\equiv g(V_M),
\end{equation}
where \Cref{thm:eigopt} ensures that $0\in[g(V_m),g(V_M)]$.
Since $g(V_m)$ and $g(V_M)$ have opposite signs
(unless one of them is $0$, in which case the solution is trivial),
we can search along a smooth curve $V(t)$ that connects $V_m$ and $V_M$
over the Grassmann manifold to find the solution where $g(V)=0$.
For example, let 
\begin{equation}\label{eq:vt}
V(t) = \mbox{orth}(t\cdot V_M + (1-t)V_m) 
\quad\text{for $t\in[0,1]$}.
\end{equation}
Then $g(V(0))=g(V_m)<0$ and $g(V(1))=g(V_M) >0$, so we can apply bisection 
to  the root-finding problem 
\begin{equation}\label{eq:grt}
    g(V(t))=0 \text{\quad with\quad $t\in[0,1]$}
\end{equation} 
to find the root $\hat t$ such that $g(V(\hat t))=0$.
Here, we assume $V(t)$ in~\eqref{eq:vt} has a full rank $r$ for all $t\in[0,1]$,
otherwise, more sophisticated treatment is required, which is beyond the scope of this work.

Once the root $\hat t$ of~\eqref{eq:grt} is found, we can construct
the FPCA solution as $U_*= [U_1,U_2V(\hat t)]$ using~\eqref{eq:usv}.
It is straightforward to verify that $U_*$ satisfies the fairness condition via~\eqref{eq:rootfind0}
and the optimality via~\eqref{eq:eigopt}.
\item 
The overall complexity of Algorithm~\ref{alg:eigsumopt} is 
\begin{equation}\label{eq:complex}
 O(n^3 (\log_2(\text{tol}))^2 + m n^2),
\end{equation}
where tol is the absolute error tolerance of $t_*$ for
the eigenvalue optimization in step~\ref{step:find_t}.
This complexity consists of the following key components: 
\begin{itemize}
    \item The SVD of $A$ and $B$ 
    requires $O(m_{1} n^2) +O(m_{2} n^2)$ time complexity in step~\ref{step:singular_values}.
    Recall that the SVD of an $m \times n$ matrix has a complexity 
    $O(m n^2)$ \citep[pp.493]{Golub:2013}. 
    In addition, constructing $H_{\!_A}$ and $H_{\!_B}$ by~\eqref{eq:HAHBdef} 
    requires $O(m_{1} n^2)+O(m_{2} n^2)$ operations;

    \item The eigenvalue optimization by Brent's method in step~\ref{step:find_t},
    with an absolute error tolerance tol, costs 
    $O([\log_2(\text{tol})]^2)$ 
    evaluations of $\phi(t)$ over the interval $[0,1]$;
    see~\cite[Sec.5.4]{Brent:2013}. 
    Each evaluation of $\phi(t)$ requires the solution of a
    symmetric eigenvalue problem of size $n$, 
    with a complexity $O(n^3)$.
    
\end{itemize}

In comparison, the SDR-based algorithm by~\cite{Samadi:2018} 
has a complexity 
$O \left(n^{6.5} \log(1/\mbox{tol}) \right)$,
if the SDP is solved by conventional convex optimization,
and $O \left({n^3}/{\text{tol}^2} \right)$,
if it is solved by the multiplicative weight (MW) update method,
where $\text{tol}$ is the error tolerance of the SDP and LP.
Those complexities are much higher than that of~\eqref{eq:complex}. 

\end{enumerate}

\section{Experiments} \label{sec:experiment}
In this section, we demonstrate the performance of EigOpt
on real-world datasets and compare the performance with those of the standard PCA and the SDR-based  FPCA.\footnote{
The code for standard PCA and EigOpt is available at \url{https://github.com/JunhuiShen/Fair-PCA-via-EigOpt}. \\
The code for the SDR-based FPCA algorithm is available at \url{https://github.com/samirasamadi/Fair-PCA}.
}

Experiments were conducted on a MacBook Pro
with a 12-core M2 Max processor @3.49GHz,
32GB of RAM, and 48MB of L3 cache.

\subsection{Datasets}

\paragraph{Bank Marketing (BM).}
This is a dataset,
introduced by~\cite{Moro:2014}, to analyze the success of direct marketing campaigns
for promoting term deposit subscriptions at a Portuguese bank.\footnote{\url{https://archive.ics.uci.edu/dataset/222/bank+marketing}} 
The dataset is divided into two age groups: 
$A \in \RR^{810 \times 16}$, representing younger individual, and
$B \in \RR^{44401 \times 16}$, representing older ones. It has since been widely used in fairness research, 
such as clustering~\citep{Bera:2019} and PCA~\citep{Kleindessener:2023, pelegrina:2024}.

\paragraph{Default of Credit Card Clients (DCC).}
This dataset was introduced in~\cite{Yeh:2009} to investigate default payment behavior in Taiwan.\footnote{\url{https://archive.ics.uci.edu/dataset/350/default+of+credit+card+clients}} 
The data is divided by education level into two groups: 
$A \in \RR^{10599 \times 23}$ represents graduate degree holders 
and 
$B \in \RR^{19401 \times 23}$ represents 
other education levels.  It has been used in 
the fair PCA study~\citep{Samadi:2018, Kamani:2022, Pelegrina:2022, pelegrina:2024}. 

\paragraph{Crop Mapping (CM).}
This dataset is from the UCI Machine Learning Repository.\footnote{\url{https://archive.ics.uci.edu/dataset/525/crop+mapping+using+fused+optical+radar+data+set}}. It was collected in Manitoba, Canada for cropland classification \citep{Femmam:2022, Vanishree:2022}.
The dataset is divided by crop type into two groups:
$A \in \RR^{39162 \times 173}$ represents corn and 
$B \in \RR^{286672 \times 173}$  represents other crops. 

\paragraph{Labeled Face in the Wild (LFW).}
This database contains face photographs and is used for studying unconstrained face recognition~\citep{Huang:2008, Schlett:2022, Taigman:2014, Wang:2023}.\footnote{\url{https://vis-www.cs.umass.edu/lfw/}}
It is  often  used in fairness research, 
particularly in studies on the fair PCA~\citep{Pelegrina:2022, Samadi:2018, Tantipongpipat:2019} for gender-based analysis.
The dataset is divided into 
two groups, with $A \in \RR^{2962 \times 1764}$ for females and 
$B \in \RR^{10270 \times 1764}$ for males.

\subsection{Experiment results}

\begin{example}  \label{eg:PCAvsFPCA_error}
{\rm 
In this example, we compare  the FPCA with the standard PCA 
in terms of the {\em reconstruction error} and {\em reconstruction loss} 
of the computed basis matrix $\widehat U_r\in\RR^{n\times r}$ for dimension reduction. 
EigOpt (\Cref{alg:eigsumopt}) is used to compute the 
solution of the FPCA~\eqref{eq:trminmax}.

Figure~\ref{fig:error+loss}(a) reports the overall reconstruction error of the data matrix 
$M$, measured by $\|M- M \widehat{U}_r \widehat{U}_r^{T}\|_F^2$, 
as a function of the reduced dimension $r$. 
As expected, the errors are reduced monotonically with increasing $r$.
We observe that the FPCA  
exhibits slightly larger reconstruction errors
than the standard PCA, with increases ranging from about $0.01\%$ to $20.44\%$ across datasets.
This reflects the trade-off between accuracy and fairness in the linear dimensionality reduction.

\begin{figure}[htbp]
\begin{center}
\subfigure[Reconstruction error $\|M-M\widehat U_r\widehat U_r^T\|_F^2$]{
\begin{minipage}[t]{0.45\linewidth}
  \centering
  \includegraphics[width=8.5cm]{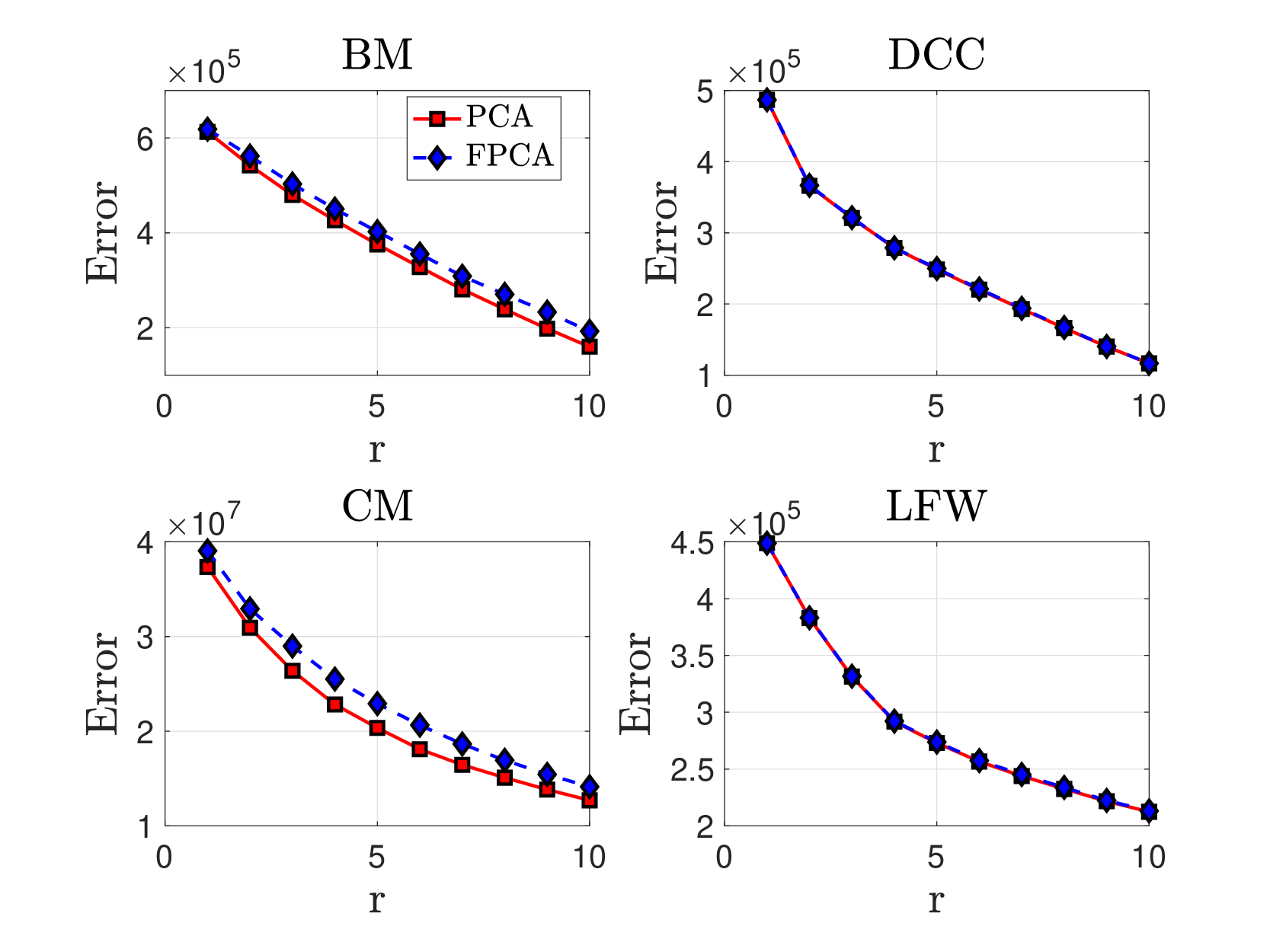}
  \end{minipage}
  }
  \
  \subfigure[Reconstruction loss for groups $A$ and $B$]{
\begin{minipage}[t]{0.45\linewidth}
  \centering
  \includegraphics[width=8.5cm]{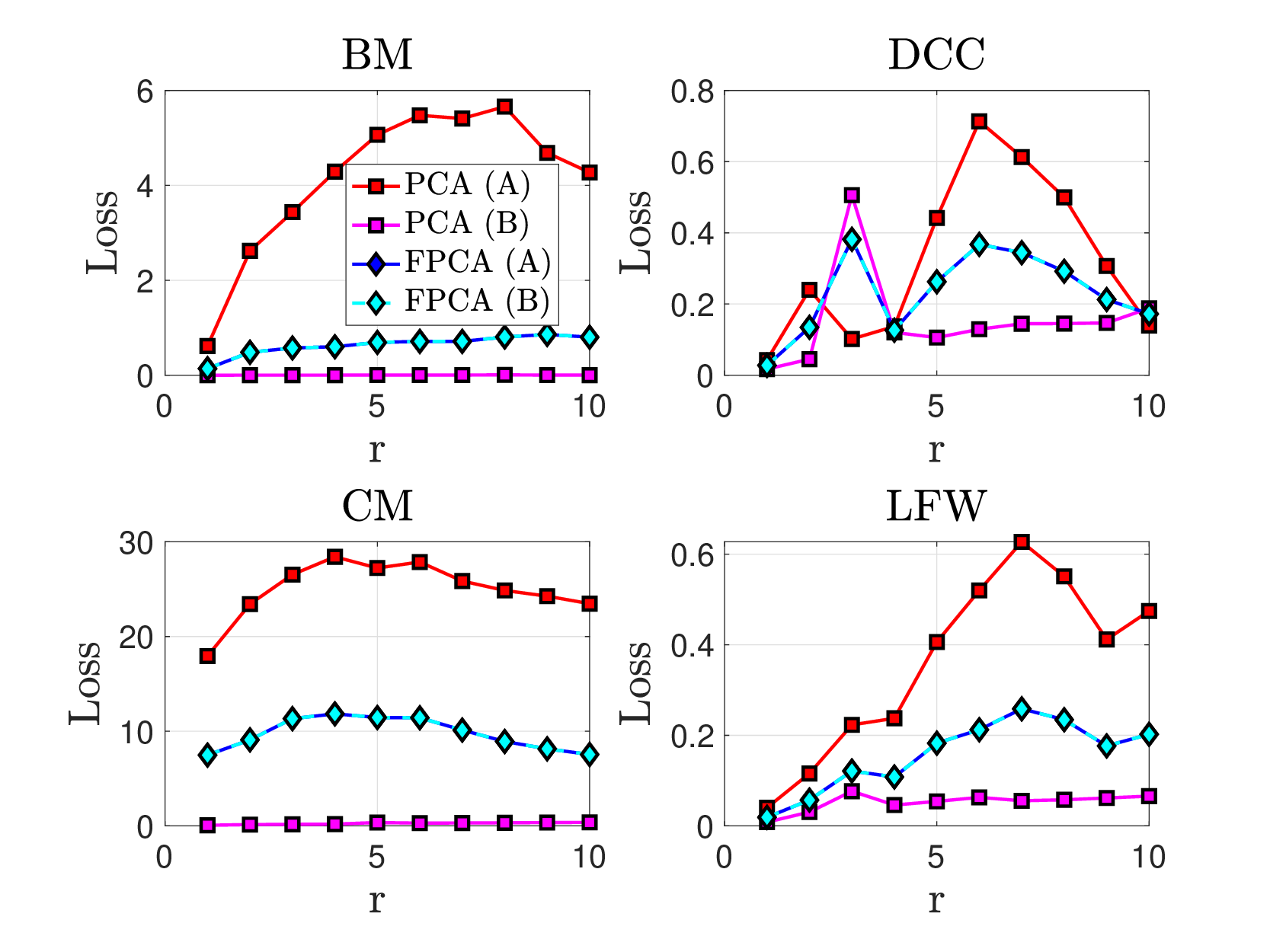}
  \end{minipage}}
\end{center}
\caption{Reconstruction error and loss in Example~\ref{eg:PCAvsFPCA_error}.}
   \label{fig:error+loss}
\end{figure}

Figure~\ref{fig:error+loss}(b) depicts the {\em average reconstruction loss} for the groups $A$ and $B$ measured respectively by 
$\Tr(\widehat{U}_r^{T} H_{\!_A} \widehat{U}_r)$ and $\Tr(\widehat{U}_r^{T} H_{\!_B} \widehat{U}_r)$.
The standard  PCA results in significant disparities, 
with the group $B$ showing much higher loss than the group $A$.
In contrast, the FPCA by~\Cref{alg:eigsumopt}
consistently achieves equity in the loss between the groups.
In all the testing cases, we observed that the FPCA solution
$\widehat U_r$ satisfy
\begin{equation}\label{eq:lossbd}
\left| \frac{\loss_{\!_A}(\widehat U_r)}{\loss_{\!_B}(\widehat U_r)} - 1\right| \leq 10^{-5},
\end{equation}
which confirms the fairness property as established in~\eqref{eq:fairobj}, and indicates high accuracy in the computed solution by~\Cref{alg:eigsumopt}.
} 
\end{example}

\begin{example} \label{eg:eigopt_vs_lp}
\rm 

We now compare the overall running time of the three PCA approaches:
(a) the standard PCA;
(b) the FPCA by EigOpt (\Cref{alg:eigsumopt}); 
and (c) the FPCA by the SRD-LP-based algorithm of ~\cite{Samadi:2018}.
The results are presented in~\Cref{tab:runtime}.
We observe that the running time of the EigOpt 
is very close to the standard PCA,
with a slowdown of only about $4.79 \%$ to $85.81 \%$.
In addition, all computed solutions of~\Cref{alg:eigsumopt} 
achieved fair reconstruction losses as in~\eqref{eq:lossbd}.
This shows that~\Cref{alg:eigsumopt} provides a reliable
way for the FPCA with a cost comparable to the standard PCA!

In contrast, the SDR-based algorithm takes significantly more time. It is about $8\times$ slower than the standard PCA and~\Cref{alg:eigsumopt}.
Moreover, the computed solutions are only sub-optimal. 
The current implementation by~\cite{Samadi:2018} used a
multiplicative weight (MW) update method for 
semidefinite programming, with a given number $T$ of iterations
(a tunable parameter).
For the test cases in~\Cref{tab:runtime}, 
the corresponding error in the loss ratio,
measured by  $\big| \frac{\loss_{\!_A}}{\loss_{\!_B}}-1\big|$,
is shown in~\Cref{fig:ratioerr}.
The SDR-based algorithm exhibits much larger deviations from the fairness criterion compared to the EigOpt, 
despite requiring more computation time.
In addition, we note that the SDR-based algorithm produces
an approximate projection matrix $P=UU^T$ rather than the basis matrix $U$.
In a number of tests, we observed that the computed $\widehat{P}$ has a rank exceeding $r$. 


\begin{table}[htbp]
\centering
\caption{ Comparison of runtime (in seconds) for 
the standard PCA, FPCA via EigOpt (\Cref{alg:eigsumopt}), FPCA via
SDR-based algorithm by~\cite{Samadi:2018}.
}
\label{tab:runtime}
\begin{tabular}{l|ccc}
\toprule 
\textbf{Dataset ($r$)} & \textbf{PCA} & \textbf{EigOpt} & \textbf{SDR} \\ 
\midrule
BM (2)    & 0.01  & 0.01  & 0.10 \\
BM (4)    & 0.01  & 0.01  & 0.10 \\
BM (8)    & 0.01  & 0.01  & 0.11 \\
\midrule 
DCC (5)   & 0.01  & 0.01  & 0.10\\
DCC (10)  & 0.01  & 0.01  & 0.10\\
DCC (15)  & 0.01  & 0.01  & 0.10\\
\midrule 
CM (30)   & 1.19  & 1.27  & 13.73\\
CM (60)   & 1.05  & 1.23  & 12.97\\
CM (120)  & 1.14  & 1.26  & 13.12\\
\midrule 
LFW (50)  & 4.82  & 6.68  & 53.99\\
LFW (100) & 5.79  & 6.95  & 56.86\\
LFW (200) & 4.68  & 8.71  & 58.20\\
\bottomrule
\end{tabular}
\end{table}

\begin{figure}[htbp] 
    \centering
    \includegraphics[width=0.6\linewidth]{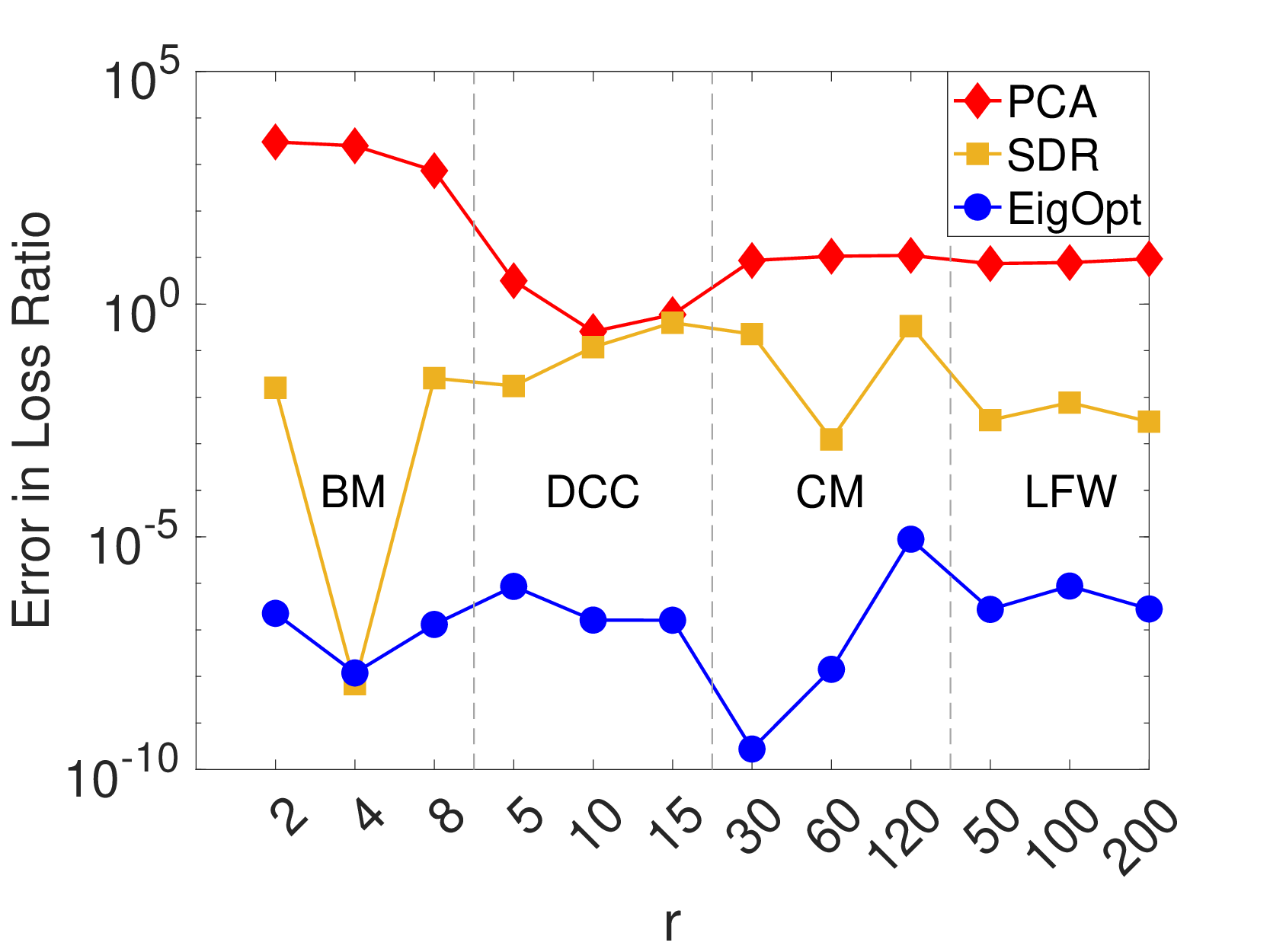}
    \caption{Error in Loss Ratio $\Big| \frac{\loss_{\!_A}}{\loss_{\!_B}}-1\Big|$
    } 
    \label{fig:ratioerr}
\end{figure}
\end{example}

\section{Concluding remarks} \label{sec:conclusion}
We presented a novel eigenvalue optimization approach for solving the FPCA problem \eqref{eq:minfdef} by uncovering a hidden convexity through a reformulation of the problem as
an optimization over the joint numerical range.
Experiments demonstrated that the proposed method is efficient, reliable, and easy to implement and reduced the computational cost of the FPCA to levels comparable to the standard PCA. 

The proposed algorithm can be naturally extended to handle the FPCA
model~\eqref{eq:minfdef} to three subgroups since 
the convexity of the joint numerical range still applies
to  three symmetric matrices, as long as the matrix size
$n > 2$~\citep{Gutkin:2004}. 
However, it remains open to how to generalize the method for four or more groups,
where the convexity of the joint numerical range no longer holds.

\bibliographystyle{apalike}
\bibliography{refs}

\newpage

\end{document}